%% file: AISTATS2026PaperPack/aistats_main.tex
\documentclass[twoside]{article}

\usepackage[preprint]{aistats2026}
\usepackage{amssymb,graphicx}
\usepackage{amsmath}
\usepackage{amsthm}

\usepackage{bm}
\usepackage{bbm}
\usepackage{commath}
\usepackage{hyperref}
\usepackage{cleveref}
\usepackage{nicefrac}
\usepackage{physics}
\usepackage{subcaption}
\usepackage{mathtools}
\usepackage{algpseudocode}
\usepackage{algorithm}
\usepackage{tikz}
\usepackage{tikz-3dplot}

\input{notations}

\usepackage[suppress]{color-edits}
\addauthor[sz]{sz}{red}

\begin{document}



\twocolumn[

\aistatstitle{Learning Under Moral Hazard with Instrumental Regression and Generalized Method of Moments}

\aistatsauthor{ Shiliang Zuo }

\aistatsaddress{ szuo.rs@gmail.com } ]

\begin{abstract}
Machine learning has become increasingly popular in informing data-driven policy-making. Policies influence behavior in individuals or populations, and ideally, through observational signals, policy-makers learn which policies are effective. However, in many settings, individual actions cannot be perfectly observed. This issue, known in economics as moral hazard, poses a significant challenge. In this work, we study the foundational multitasking principal–agent contract design problem and demonstrate how instrumental regression and the generalized method of moments (GMM) estimator can be used to estimate or learn a good contract. As a bonus result, we also give a uniformity characterization of the shape of the optimal contract. 
\end{abstract}

\input{MAINBODY}
\end{document}

%% file: notations.tex
\newcommand{\cD}{\mathcal{D}}

\newcommand{\bbR}{\mathbb{R}}

\newcommand{\Ex}{\mathbb{E}}

\newcommand{\eps}{\varepsilon}

\newcommand{\IGNORE}[1]{}

\newcommand{\tildeX}{\tilde{X}}
\newcommand{\tx}{\tilde{x}}
\newcommand{\tX}{\tilde{X}}

\newtheorem{theorem}{Theorem}
\newtheorem{lemma}{Lemma}
\newtheorem{remark}{Remark}
\newtheorem{assumption}{Assumption}

\newtheorem{cor}{Corollary}
\newtheorem{example}{Example}
\newtheorem{condition}{Condition}
\newtheorem{proposition}{Proposition}

\DeclarePairedDelimiter{\ceil}{\lceil}{\rceil}

\algnewcommand{\LineComment}[1]{\State \(\triangleright\) #1}

\renewcommand{\ip}[2]{\langle #1, #2 \rangle}

\let\tilde\widetilde

\newcommand{\diag}{\mathtt{diag}}

\newcommand{\tO}{\widetilde{O}}

%% file: MAINBODY.tex
\input{sec/introduction}

\input{sec/related}

\input{sec/model}

\input{sec/uniform}

\input{sec/GMM}

%\input{sec/learning}

\input{sec/repeated}

\input{sec/strategic_classification}
\input{sec/conclusion}

%Bibliography
\bibliographystyle{apalike}  
\bibliography{references}

\clearpage
\appendix
\thispagestyle{empty}

% Supplementary material: To improve readability, you must use a single-column format for the supplementary material.
\onecolumn
\aistatstitle{
Supplementary Materials}

\input{sec/app_naiveEst}

\input{sec/app_diversity}

\input{sec/experiments}

%% file: sec/introduction.tex
\section{INTRODUCTION}
\label{sec:intro}
Machine learning has become a powerful tool for data-driven decision-making across many real-world applications, including ridesharing (\cite{qin2022reinforcement}), education policy-making (\cite{hilbert2021machine}), healthcare policy-making (\cite{ashrafian2018transforming}), and credit scoring (\cite{fuster2022predictably, hurley2016credit}), among others. A central assumption in most machine learning methods is that the causal relationship between inputs and outcomes is well-specified and observable. However, in many real-world policy-making scenarios, the dependence between signals and outcomes is often complex and difficult to model. In particular, in economic and strategic environments, this assumption breaks down: data are generated by self-interested agents, whose underlying actions drive outcomes but remain hidden. As a result, the true causal links between observed signals and unobserved behavior are confounded, posing fundamental challenges for learning. 

Consider an educational setting where we are interested in measuring students’ skills to predict their future success (however defined). While some information is observable—such as educational background—many important features remain hidden, such as problem-solving or critical thinking skills. These latent skills can only be indirectly inferred through noisy signals like standardized test scores. In this case, the covariates of interest are not perfectly observed. Moreover, policymakers may wish to design better education policies based on students’ data, but doing so requires addressing the fact that underlying behaviors and abilities are only imperfectly captured.

As another example, consider a vehicle insurance company seeking to design policies for its customers. A policyholder’s driving behavior is influenced by the insurance plan: some individuals may become more cautious, while others may drive more recklessly, depending on the policy they participate in. The insurer’s objective is to shape behavior in a way that maximizes profit. Yet, driving behavior cannot be directly observed; instead, the insurer only has access to noisy outcomes such as accident reports. Despite this limitation, the company must still rely on historical data to design effective insurance policies.

This type of challenge is known as moral hazard in economics (\cite{holmstrom1979moral}). Moral hazard arises when an agent’s actions are hidden from the principal, making it difficult to align incentives. In this work, we initiate the study of data-driven methods for decision-making under moral hazard, where optimal policies must be learned from observational signals that depend on unobserved actions.

We focus on the multitasking principal–agent problem, a canonical model of contract design capturing the salient feature of moral hazard (\cite{holmstrom1991multitask}). Contract design is an important topic in economics: it captures how a principal can incentivize an agent to exert effort in working on tasks by linking payments to observable signals. In the multitasking principal-agent problem, the principal hires an agent to perform several tasks. The principal's utility depends on the agent's hidden action; however the effort the agent puts into each task can only be measured through a noisy signal. 

In this work, we study how to learn effective contracts for the principal under such information constraints. Our approach combines tools from economics, econometrics, and machine learning, and in particular demonstrates how econometric methods such as instrumental regression and the generalized method of moments (GMM) can be adapted to address endogeneity and measurement error in this strategic environment. 

Apart from our learning results, we also provide a fairness characterization of optimal contracts. In many real-world settings, the principal interacts with multiple agents under a common contractual framework; for instance, franchising arrangements \cite{bhattacharyya1995double} or revenue-sharing platforms such as YouTube, where content creators receive 55\% of advertising revenue while the platform retains 45\% \footnote{See \cite{YouTubeShare}. }. While one might expect the principal to customize terms for different agents, we show that when the agent’s cost function exhibits homogeneity, the optimal contract depends only on its degree of homogeneity. This implies that uniform contracts can simultaneously maximize the principal’s utility and reinforce fairness across agents.

% Apart from the results on learning, we also gave a fairness characterization on the optimal contracts. In many real-world scenarios, the principal engages with multiple agents under the same contractual framework. A prominent example is franchising, where the principal (franchiser) contracts with multiple franchisees to ensure consistent effort and revenue generation \cite{bhattacharyya1995double}. Another example is e-commerce platforms such as YouTube, where content creators earn a share of advertising revenue. According to official documentation \cite{YouTubeShare}, YouTube content creators receive 55\% of the total advertisement revenue, while the remaining 45\% is retained by the platform. While one might expect the principal to tailor contracts for different agents to maximize utility, it is often observed that contracts remain uniform across agents. To answer this question, we show that when the agent’s cost function exhibits a certain degree of homogeneity and the principal’s utility is linear across tasks, the optimal contract is necessarily uniform across agents. Specifically, the optimal contract depends on the agent's cost function \emph{only through} it's the homogeneity degree. This result has strong practical implications: it suggests that the principal does not need to differentiate contracts to maximize utility, reinforcing fairness in contract design. Our findings are similar in spirit to the work of \cite{bhattacharyya1995double}, who examine a single-task setting with double moral hazard. In contrast, our study focuses on multitask environments. 

\subsection{Contributions} In this work, we study multitasking principal-agent. We first give a \emph{fairness characterization} of the optimal contract. We show that under a homogeneity assumption on the agent’s cost function, the optimal linear contract depends only on the degree of homogeneity and is uniform across agents. This provides theoretical insight and practical justification for standardized contracts in settings like franchising and online platforms.

We then study the multitasking principal-agent problem from a \emph{machine learning} perspective, we show how tools from econometrics, particularly instrumental regression and the generalized method of moments (GMM), can be applied in this setting in designing an effective contract. To our knowledge, this is the first work to examine the multitasking principal–agent problem through a machine learning lens. First, we identify contract learning under moral hazard as a regression problem with measurement error and show how instrumental regression, together with the generalized method of moments (GMM) estimator, can be used to recover the unknown parameters and thereby identify the optimal contract. Second, we demonstrate that when repeated signals are available and agents are sufficiently diverse, the principal can achieve significantly faster convergence rates. 

%% file: sec/related.tex
\subsection{Related Work}

\paragraph{Machine Learning in Strategic Contexts} Machine learning increasingly operates in environments where data are generated by self-interested agents. Recent work introduced the study of strategic classification (\cite{hardt2016strategic}) and performative prediction (\cite{perdomo2020performative}), showing how individuals may manipulate features in response to predictive models. Following this, researchers have highlighted endogeneity issues in strategic settings and developed instrumental regression approaches for regression with strategic responses (\cite{harris2022strategic}), online learning and bandits (\cite{della2023online}), and reinforcement learning with hidden actions (\cite{yu2022strategic}). More broadly, recent work emphasizes the causal perspective in machine learning systems (\cite{horowitz2023causal}, \cite{miller2020strategic}), motivated by real-world applications where decisions shape behavior and the data reflect these strategic interactions. 

\paragraph{Contract Theory} Contract theory studies a principal-agent problem where the principal must incentivize the agent to exert effort through contracts. Some early important works in the economics community include \cite{holmstrom1979moral, holmstrom1982moral, holmstrom1991multitask}. The work by \cite{holmstrom1991multitask} studies the multitask principal-agent problem, which serves as the starting point of the current work. Recently, computational aspects of contract design have been studied by many works in the computer science community. For example, \cite{dutting2019simple, duetting2024multi, dutting2022combinatorial} study contract design from a combinatorial perspective; various other works \cite{ho2014adaptive, zuo2024harnessing, guruganesh2024contracting, Zhu2022TheSC} study learning algorithms for contract design. Another line of work seek to understand the worst-case guarantee of contracts (e.g. \cite{carroll2015robustness}), and show that linear contracts have remarkable worst-case performance guarantees. Apart from these, contract theory has also appeared in application domains such as signal processing (\cite{jain2023adaptive}).

% More generally speaking, some recent work study the causal aspect of machine learning, in particular in societal decision-making contexts\cite{horowitz2023causal}. 

\paragraph{Bandit Problems} Bandit problems are a type of online learning problem with partial feedback. Typically, bandit algorithms need to balance exploration (experimenting with unknown actions) and exploitation (choosing actions whose reward is estimated to be higher). Research on bandit problems is too broad to cover here, but for an overview see \cite{lattimore2020bandit}. Recently, a number of papers have sought to understand the performance of the pure exploitation (i.e., greedy) algorithm in bandit problems. In particular, for linear bandits, the success of the greedy algorithm is explained under the framework of smoothed analysis \cite{bastani2021mostly, kannan2018smoothed, sivakumar2022smoothed}. 

\paragraph{Instrumental Regression and Measurement Error Models. }
Measurement error models are a well-studied topic in econometrics and statistics, as noisy or imperfectly observed covariates can lead to biased and inconsistent estimates. A standard remedy is the use of instrumental variables, where an observed variable correlated with the true covariate but independent of the error serves as a proxy. The generalized method of moments (GMM) framework further offers a flexible way to construct consistent estimators in the presence of endogeneity. These methods are well established in econometrics; see \cite{alma99653821012205899,fuller2009measurement} for textbook treatments. % More broadly, IV and GMM are powerful tools for addressing causal inference with endogenous variables. Recent work in computer science has increasingly explored causal perspectives in machine learning, including strategic classification \cite{miller2020strategic}, performative prediction \cite{perdomo2020performative}, and bandit problems with unobserved confounding \cite{kallus2018instrument}. 

% \paragraph{Exploration-Free Algorithms and Smoothed Analysis}In the bandit community, there is recently a interest in studying the behavior of greedy algorithms. In particular, in linear bandits, the success of the greedy algorithm can be explained via smoothed analysis \cite{sivakumar2022smoothed, bastani2021mostly, kannan2018smoothed}. 

%% file: sec/model.tex
\section{\MakeUppercase{The Multitasking Principal-Agent Problem}}
In the multitasking principal-agent problem, there are two parties, a principal and an agent. The principal asks the agent to complete several tasks. The agent exerts effort across $d$ dimensions, each representing a distinct task (so there are $d$ tasks). We denote the agent’s effort by the vector $a \subset (\mathbb{R}^+)^d$. 

The agent incurs a private cost $c(a) \in \bbR^+$ when exerting the effort vector $a$. We will assume the cost function $c(a)$ is strictly increasing, strictly convex, and continuously differentiable. While the principal cannot directly observe or verify the exact effort level $a$, she can observe a signal $x \in \bbR^d$; the signal can be interpreted as a noisy measurement of the agent's true effort vector, in particular, the signal $x_i$ can represent a noisy measurement for the true effort $a_i$ in task $i$. We shall assume the signal is unbiased:
\[
\Ex[x|a] = a. 
\]
The unbiasedness is a common assumption (e.g. \cite{holmstrom1991multitask}). 

The principal incentivizes the agent to exert effort through a contract. We focus on linear contracts, parameterized by $\beta \in (\mathbb{R}^+)^d$. Under the linear contract $\beta$, when the observed signal is $x$, the payment transferred from the principal to the agent is $\ip{\beta}{x}$; in expectation, the transfer is equal to $\ip{\beta}{a}$. 

\paragraph{Agent's Response} The agent's utility is his expected payment minus his private cost. The agent selects an action $a = a(\beta)$ that maximizes their expected utility given the contract terms: 
\[
a(\beta) = \arg\max_{a} \ip{\beta}{a} - c(a). 
\]

\paragraph{Principal's Utility}

The agent's effort $a$ gives a noisy private benefit $y(a)$ to the principal. We assume this private benefit takes a linear form
\[
\Ex[y(a) | a] = \ip{\theta^*}{a}. 
\]
The principal's expected utility is the expected private benefit minus the expected payment to the agent. In other words, when the contract offered is $\beta$ and the agent best responds with the action $a$, the principal's expected utility is
\[
u(\beta; a) = \ip{\theta^*}{a} - \ip{\beta}{a}. 
\]
We may also write $u(\beta) := u(\beta, a(\beta))$; here the variable $a$ is suppressed and implicitly understood as the agent's best response to $\beta$. The principal's optimal contract $\beta^*$ is then the contract maximizing her expected utility:
\[
\beta^* \in \arg\max_{\beta} u(\beta). 
\]
% In the expression below, letting $a$ be the best response to the contract $\beta$, the principal's expected utility is then
% \[
% u(\beta) = \Ex[y(a)] - \ip{\beta}{a}. 
% \]

% Then, under the contract $\beta$, assuming the agent takes his best action $w$, the principal's expected utility is
% \[
% \ip{\theta^*}{a} - \ip{\beta}{a}, 
% \]
% i.e., the expected private benefit minus the expected payment to the agent. The optimal contract is the contract maximizing the principal's expected utility. 

\paragraph{Interaction} The interaction protocol can be summarized as follows. 
\begin{enumerate}
	\item The principal posts a contract $\beta \in \mathbb{R}^d$
	\item The agent best responds with a private effort (action) $a \in \mathbb{R}$, and incurs private cost $c(a)$
	\item A noisy signal $x$ is generated and that $\Ex[x|a] = a$
	\item The agent's expected utility is $\ip{\beta}{a} - c(a)$
	\item The principal's private realized benefit is $y$ with $\Ex[y|a] = \ip{\theta^*}{a}$, the expected utility $u(\beta; a) = \ip{\theta^*}{a}- \ip{\beta}{a}$. 
\end{enumerate}
Note that from the principal's point of view, apart from the decision variable $\beta$, only the signal $x$ and the realized private benefit $y$ are observed; the agent's true effort $a$ is not observed. The causal relation between the variables is summarized in \Cref{fig:causalFig}. 

\szcomment{begin revision content. Add discussion here. }
Before we present our main technical results, let us first clarify how our model captures the essential features of moral hazard. In its modern economic usage, moral hazard refers to settings where hidden actions create incentive misalignment between contracting parties. Our multitasking contract design problem exhibits the core characteristics of moral hazard: (1) the agent's action is unobservable and non-contractible, (2) actions affect outcomes but only noisy signals are observed, and (3) contracts must be designed based on observable signals to incentivize hidden actions.

These features are present in our motivating examples (\Cref{sec:intro}). In the insurance setting, driving behavior is hidden from the insurer, who observes only noisy signals such as accident reports and claims costs; the insurer must design policies that incentivize safe driving to maximize profit. Similarly, in the education setting, student effort and true skills are latent, while policymakers observe only noisy test scores and long-term outcomes; they must design policies that incentivize appropriate effort allocation across different subjects or skill areas.

As a final remark, we note that the unbiasedness assumption on signals specifies the information structure but does not make actions observable or contractible. The moral hazard remains intact: the principal cannot directly observe or write contracts contingent on the agent's effort $a$. Although unbiasedness is a common assumption in the multitasking literature (e.g. \cite{holmstrom1991multitask}), extending our results in this work to accommodate possibly biased signals is also a valuable future direction. 
\szcomment{end revision content. }

%% file: sec/uniform.tex
\section{\MakeUppercase{Uniformity of the Optimal Linear Contract}}

We first give a characterization of the optimal linear contract. We present a uniformity result, which states that if the agent's cost function exhibits a certain degree of homogeneity (i.e. a consistent return to scale), then the optimal linear contract depends on the cost function only through its homogeneity degree. 
% \begin{definition}
% The cost function $c$ is said to satisfy the Inada conditions if the following holds. 
% \begin{enumerate}
% \item $c$ is convex on its domain. 
% \item $\lim_{a_i \rightarrow 0} \partial c(a) / \partial a_i = 0$. 
% \item $\lim_{a_i \rightarrow +\infty} \partial c(a) / \partial a_i = +\infty$. 
% \end{enumerate}
% \end{definition}
% \begin{remark}
% The Inada condition is fairly standard. It essentially means the marginal cost goes from $0$ to $+\infty$ as $a_i$ goes from $0$ to $+\infty$. 
% \end{remark}

\begin{assumption}
\label{assump:cost}
The cost function of the agent is homogeneous of degree $k$. In other words, for any $\rho > 0$ and $a\in (\bbR^d)^+$, $c(\rho a) = \rho^k c(a)$. 
\end{assumption}
\begin{remark}
The homogeneity assumption was also made in a related line of work in strategic classification \cite{shavit2020causal, dong2018strategic}. Another work on learning from revealed preferences in Stackelberg games also make a homogeneity assumption \cite{roth2016watch} (see the subsequent \Cref{remark:watch-and-learn-compare}). 
\end{remark}
% \begin{remark}
% The parameter $k$ is also known as the return to scale parameter. 
% \end{remark}

% \begin{example}
% A natural class of cost functions which are homogeneous is the class of CES disutility functions. 
% \end{example}

% We also assume the private benefit to the principal depends only on the agent's true effort vector $a$ through a linear model. 
% \begin{assumption}
% There exists $\theta^*\in (\bbR^d)^+$, such that
% \[
% \Ex[ y(a) ] = \ip{\theta^*}{a}. 
% \]
% \end{assumption}
% The value $\theta^*_i$ can be interpreted as the marginal utility of task $i$ to the principal. 

We show the optimal linear contract depends on the shape of $c(\cdot)$ only through the degree of homogeneity. The proof is very simple, yet to the best of our knowledge, it has not appeared in prior work. 

\begin{theorem}
\label{thm:uniformity}
The principal's optimal contract is $\beta^* = \theta^* / k$. 
\end{theorem}

\begin{proof}
Suppose the principal posts the contract $\beta$. By the first-order condition of the agent's best response, we have:
\[
\beta = \nabla c(a). 
\]
Then, the principal's utility function when posting the contract $\beta$, as a function of the best response $a$ can be written as 
\begin{align*}
\ip{\theta^*}{a} - \ip{\beta}{a} &= \ip{\theta^*}{a} - \ip{\nabla c(a)}{a} \\
&= \ip{\theta^*}{a} - k c(a),
\end{align*}
where we used Euler's theorem on homogeneous functions. Notice that the principal's utility function is a concave function in the agent's hidden effort $a$. Taking derivative with respect to $a$, we have
\[
\theta^* = k\nabla c(a). 
\]
Combined with the agent's first-order condition, we have that at the optimum we must have 
\[
\beta^* = \theta^* / k. \qedhere
\]
\end{proof}

\begin{remark}
The uniformity result is particularly attractive from a practical perspective. It implies that the principal can achieve her maximum utility and ensure fairness simultaneously. In particular, there is no need to discriminate against agents whose productivity may differ (as long as their return to scale parameter $k$ remains the same). This standardization is particularly relevant in industries like franchising or E-commerce, where the principal (the franchiser or the E-commerce platform) contracts with \emph{multiple} agents (franchisees, E-commerce workers / content creators). 
\end{remark}

\begin{remark}
The uniformity result here is similar in flavor to \cite{bhattacharyya1995double}; the difference is that their work considers a double moral hazard setting with a single task. They also make a similar homogeneous degree assumption, but of course, in the single-dimensional case, a homogenous function can only take the polynomial form. 
\end{remark}

% The results may also be particularly relevant in E-commerce scenarios. For example, consider YouTube, the video-sharing platform. According to this online article (\cite{YouTubeShare}), creators receive 55\% of the advertising revenue while the platform retains 45\%. While different creators have different abilities in generating revenue, the contract remains the same for different content creators. 

% As a special case, the results explain the widespread 50-50 split rule seen in many scenarios (e.g., \cite{allen1993transaction}). Specifically, when the agents have a quadratic cost function ($k = 2$), the results indicate that the 50-50 split is precisely the optimal linear contract which ensures maximum utility for the principal. 
\begin{remark}
\label{remark:watch-and-learn-compare}
The recent work \cite{roth2016watch} in fact studied a very similar setup with the same homogeneity assumption. Their work studied optimization algorithms in Stackelberg games from revealed preferences. As a special case, they gave a discussion on how their algorithm can be applied to the multitask principal-agent problem (though their work did not explicitly identify their problem as such). They also make a homogeneity assumption; however, the uniformity result that we presented here seemed to have somehow escaped their analysis. 
\end{remark}
%The problem of learning the optimal contract is also studied in this work, however, the approach is fundamentally different from theirs. This is because the current work also identified a ``uniformity'' result, which seemed to escape their analysis. The problem of estimating and learning the optimal contract in this work uses instrumental regression with the generalized method of moments method and is based on this ``uniformity'' result, whereas in their work the problem is solved via a two-stage optimization algorithm. 

% As a remark, the same problem setup in fact has been considered in the prior work \cite{roth2016watch}. However, this uniformity result seems to have escaped their analysis. As a consequence, while in both work the problem of learning the optimal contract is studied, the approach is quite different. 

% In comparison to previous work by Roth et al. (2016) \cite{roth2016watch}, the same problem setup has been considered, but the uniformity result seems to have escaped their analysis. Consequently, although both studies address the challenge of learning the optimal contract (this is the following sections in the current work), the methodologies employed are markedly different.

% Finally, the uniformity result presented here follows closely in spirit to a result in the important paper \cite{bhattacharyya1995double}. The difference is that in their setting, they consider a single task but with double moral hazard (i.e., effort from the principal is also required, and effort from both sides is unverifiable). 

%% file: sec/GMM.tex
\section{\MakeUppercase{Estimating and Learning the Optimal Contract: an Instrumental Regression Approach}}
\label{sec:contractIV}

\begin{figure}[t]
    \centering
    \includegraphics[width=.50\textwidth]{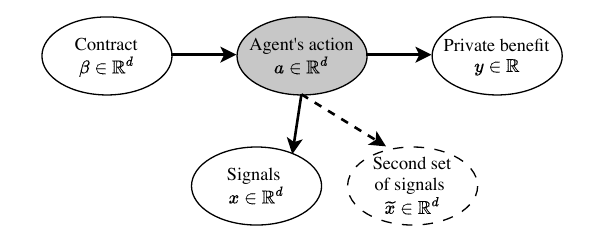} % Replace 'example-image' with your image file name
    \caption{Causal relationship between variables. The agent's response is shaded indicating it is unobserved by the principal. The dashed item is a second set of observed signals, corresponding the the scenario studied in \Cref{sec:repeated} when repeated observations are available. }
    \label{fig:causalFig}
\end{figure}

In the previous section we gave a characterization of the optimal contract and showed that even for different agents with different cost functions, the optimal contract remains the same as long as the homogeneity degree remains the same. Now, we study the problem of estimating or learning the optimal contract when the parameter $\theta^*$ is unknown. Informally, when the parameter $\theta^*$ is unknown, the principal lacks precise knowledge of the ``importance" of each task. For instance, when a firm's owner (the principal) hires a manager (the agent) to oversee multiple tasks, the owner may be uncertain about the exact contribution of each task to the firm's overall performance. In such cases, the principal must estimate or learn $\theta^*$ using observational data. 

We assume the principal repeatedly interacts with potentially different agents, and denote the observed data in the $t$-th interaction as $(\beta_t, x_t, y_t)$, representing the contract, observed signal, and private benefit, respectively; the agent's private effort $a_t$ is unobserved by the principal. In a single round of interaction, the relationship between the variables is depicted in \Cref{fig:causalFig}. 
Similar to the previous section, it is assumed the observed signals for each task is unbiased:
\[
\Ex[x_t | a_t] = a_t,
\]
and that the total revenue satisfies a linear relation:
\[
\Ex[ y_t |a_t ] = \ip{\theta^*}{a_t}. 
\]

In repeated interactions, we allow the agent in each round to be of different types. In particular, different agents may have different cost functions and react differently to the same contract. 
However, we assume that the homogeneity degree parameter $k$ is known and treated as a constant. In other words, the return to scale parameter is a constant for different agent types. \footnote{While one could also study the scenario where $\theta^*$ is known but $k$ is unknown, this case seem less interesting from a technical perspective; since $k$ is a single-dimensional parameter, one can apply a one-dimensional discretization-based algorithm and infer the value of $k$ (\cite{kleinberg2004nearly}). Thus, we focus on the case where the high-dimensional parameter $\theta^*$ is unknown while $k$ is known and fixed, which we believe to be the more technically interesting problem. } Informally, this means that even though agent's may have different ``proficiency" in completing the tasks, the scaling behavior of the cost function is consistent. 

Note that in general, in particular when the homogeneity \Cref{assump:cost} does not hold, recovering the unknown parameter $\theta^*$ will not inform the shape of the optimal contract, since the agent's cost function is unknown. However, thanks to the homogeneity assumption and the uniformity result, which establishes the optimal contract as given by $\beta^* = \theta^* / k$, obtaining an accurate estimate for $\theta^*$ is equivalent to identifying an approximately optimal contract.

We consider both the offline and online setting. In the offline setting, the principal has access to data in $T$ time periods, with the observed data denoted as $(\beta_t, x_t, y_t)$; the goal is then to obtain an accurate estimate of $\theta^*$. In the online setting, the principal interacts with agents in a sequential manner, and in each round $t$ the principal chooses a contract $\beta_t$ based on historical observations; the goal here is to minimize cumulative utility loss.

We assume the contract $\beta$ belongs to some bounded set $\mathcal{B}$, and that the optimal contract $\beta^* = \theta^* / k \in \mathcal{B}$. We refer to $\mathcal{B}$ as the feasible contract set; it can represent either external factors (such as certain law or regulations) which limit the set of contracts that can be enforced, or it can represent the fact that the principal has some limited prior knowledge that the optimal contract must belong to some bounded set (for example, he may at least have an lower and upper bound of each component of $\theta^*$). 

% Note that the problem of learning $\theta^*$ bears some resemblance to the linear bandit problem, where seemingly the signal vector $x$ represents the context and the principal's private benefit $y$ represents the noisy realized reward. 

% Our problem is in fact quite different. This is because, in fact, the signal is a noisy realization of the true ``context", which is the agent's private effort. Recall that $\Ex[y(a)] = \ip{\theta^*}{a}$, so $a$ is the true ``context" vector in the linear bandit problem. Essentially, our problem faces a measurement error problem. 

% We make a mild assumption that coordinates of $\theta^*$ are bounded, so that there is some bounded set $\mathcal{B}$ in which the optimal contract $\beta^*$ belongs to. This is relatively mild. Even though the principal may not know the exact value of $\theta^*$, she may have some reasonable bound on the marginal utility of each task, i.e., on $\theta^*_i$. For example, if the principal knows that each $\theta_i$ will be bounded in the interval $[c_0, C_0]$, then we can take $\mathcal{B} = [c_0/k, C_0/k]^d$. Without loss of generality we can perform normalization and assume the following holds. 
% \begin{assumption}
% There exists some set $\mathcal{B} \subset [0,1]^d$, such that $\beta^* \in \mathcal{B}$. 
% \end{assumption}

Finally, we assume the signals $x_t$ and realized private benefits $y_t$ are conditionally subgaussian. 
\begin{assumption}
Given the agent's action $a_t$, the random signal $x_t$ and private benefit $y_t$ are $\sigma_0$-subgaussian. 
\end{assumption}

We note that standard linear regression in fact fails in recovering the $\theta^*$. Treating the problem as a regression problem, the true covariates $a$ are not observed, and only a noisy measurement $x$ is observed. Hence, our problem in fact has an errors-in-variables problem, and must be solved using techniques from measurement error models. Below we show how the generalized method of moments estimator can be used in recovering $\theta$.

\subsection{Offline Setting}
The causal relationship between $(\beta_t, a_t, x_t, y_t)$ is summarized in \Cref{fig:causalFig} and forms a measurement error model. 
It can be observed that the contract $\beta_t$ acts as a valid instrumental variable. In particular, the following moment condition is satisfied
\begin{equation}
\Ex[\beta_t (y_t - \ip{\theta^*}{x_t})] = \mathbf{0} \ (\in \mathbb{R}^d). 
\end{equation}
Therefore an instrumental regression approach and the generalized method of moments estimator can be applied. When given a sequence of offline data, the generalized method of moments estimator takes the following form: 
\begin{equation}
\hat{\theta}_T = (B^\top_T X_T)^{-1} B^\top_T Y_T. \label{eq:GMM_v1}
\end{equation}

Here $B_T \in \bbR^{T \times d}$ is a matrix with rows representing the posted contracts each round; $X_T \in \bbR^{T \times d}$ is a matrix with rows representing the signals each round; $Y_T \in \bbR^{T}$ is a column vector representing the outcomes each round. 

% Now the principal essentially faces an optimal design problem with an exploration-exploitation tradeoff.  In the following, I propose two algorithms that achieve a sublinear regret. 

\begin{proposition}
\label{prop:offline-est-1}
With probability at least $1-\delta$, the estimation error
\[
\norm{\theta^* - \hat{\theta}_T}_2 \le \frac{\sqrt{dT\log(dT/\delta)}}{\sigma_{\min} (B_T^\top X_T) }. 
\]
\end{proposition}
\begin{proof}
Denote $\gamma_t = y_t - \ip{\theta^*}{x_t}$, and collect values for $\gamma_{1:t-1}$ into the column vector $\Gamma_t$. 
\begin{align*}
\hat{\theta}_T &= (B_T^\top X_T)^{-1} B_T^\top Y_T \\
&= (B_T^\top X_T)^{-1} B_T^\top (X_T \theta^* + \Gamma_T) \\
& = \theta^* + (B_T^\top X_T)^{-1} B_T^\top \Gamma_T. 
\end{align*}
Hence
\[
\norm{\theta^* - \hat{\theta}_T}_2 \le \frac{ \norm{B_T^\top \Gamma_T}  }{\sigma_{\min} (B_T^\top X_T) }. 
\]
By a standard concentration inequality (see Appendix \ref{app:concentration}) the numerator can be upper bounded as $\sqrt{dT\log(dT/\delta)}$ with probability at least $1-\delta$; the proposition then follows. 
\end{proof}

\subsection{Online Setting}
\begin{algorithm}[t]
\caption{Use Contract as IV: Explore then Commit Algorithm}
\label{alg:perturb}
\begin{algorithmic}
% \State Set $T_1 = 2$ and $\mathcal{T}_{i} = \mathcal{T}_{i-1}^4$, epoch $i$ consists of $\mathcal{T}_i$ rounds
\State Input: Distribution $\mathcal{P}$ supported on the feasible contract space $\mathcal{B}$ satisfying \Cref{condition:randomExploration}
\State Number of exploration rounds $\tau = d\sqrt{T}$
\For{$t = 1, \dots, \tau$}
    % \State // If using basis vector for `exploration':
    % \State Set $\beta_t = \mathbf{e}_i$, here $i = \ceil{t / \sqrt{T}}$
    \State Sample $\beta_t \sim \mathcal{P}$ as the contract
    \State Observe the signal $x_t$ and realized private benefit $y_t$
    % \State // If using small random perturbation for `exploration':
    % \State Set $\beta_t = (\theta_0 + \rho_t) / k$, here $\rho_t$ is a suitable random vector with variance $\sigma_0^2$ in each coordinate
    % % \State $\hat{\theta}_t = (B_{t}^\top X_t)^{-1} B_t^\top Y_t$
    % % \State $\beta_t = \hat{\theta}_t / k + \rho_t$, where $\rho\sim N(?, ?)$
    % \State Observe the signal $x_t\in \bbR^d$ and the total revenue $y_t \in \bbR$
\EndFor
\State Record the data so far into $B_{\tau}, X_{\tau}, Y_{\tau}$
\State Use the GMM estimator and compute the estimate:
\[
\hat{\theta} = (B_\tau^\top X_\tau)^{-1} B_{\tau} Y_\tau. 
\]
\For {$t = \tau + 1, \dots, T$}
\State Set $\beta_t = \hat{\theta} / k$ as the contract
\EndFor
\end{algorithmic}
\end{algorithm}
In the previous part, we developed an estimator based on GMM for estimating the parameter $\theta^*$. This applies to the setting when offline observations are available. However, it does not directly give rise to a learning algorithm in the online setting, where the principal needs to \emph{choose} the contract each round. To give some intuition, in this setting, the principal essentially faces an optimal design problem with an exploration-exploitation tradeoff. Specifically, the matrix $B_t$, consisting of contracts $\beta_t$, is the design, and the principal must choose this matrix such that the estimate $\hat{\theta}$ becomes accurate. At the same time, the principal must ensure that the chosen contracts $\beta_t$ are close to the true optimal contract $\beta^* = \theta^* / k$ in order to guarantee a small cumulative utility loss. We define the cumulative utility loss as the different between the utility achieved with the posted contract and that of posting the optimal contract each round:
\[
u_t(\beta^*) - u_t(\beta_t). 
\]

%%%%%%%%%%%%%%%%%%%%%
\iffalse
\subsubsection{Assumptions}
We impose the following assumptions. \begin{assumption}
	\label{assump:bound}
	Let $\Theta = [0, 1]^d$. Then
	\begin{enumerate}
		\item The unknown parameter $\theta^*\in \Theta$. 
		\item Fix any agent's cost function, let $u(\beta)$ be the principal's expected utility when posting contract $\beta$. Then $u''(\beta^*) < S_0$ for some constant $S_0$. 
		\item The signal and the reward are conditionally subgaussian, in particular, $x_t | w_t$ and $y_t | x_t, w_t$ are $\sigma^2$-subgaussian. 
		% \item Fix any agent's cost function satisfying \cref{assump:cost}, and let $u(\beta)$ denote the expected utility to the principal when posting contract $\beta$. There exists some constant $c$ so that for any $\hat{\theta} \in [c_0, C_0]^{n}$, the following holds
		% \[
		% \abs{ u(\beta^*) - u(\hat{\theta} / k) } \le c\cdot \norm{\beta^* - \hat{\theta} / k}_2^2. 
		% \]
	\end{enumerate}
\end{assumption}
Both parts of the assumption are natural and relatively mild. The first part is without loss of generality, since one can apply the suitable normalization. The second part of the assumption can essentially be seen as a smoothness assumption on the principal's utility function $u$. 

Then, suppose the principal posts a contract $\beta_t$ in round $t$, by \cref{assump:bound}, the utility loss can be bounded by the square estimation error (up to constant factors):
\[
\abs{ u(\beta^*) - u(\beta_t) } \le O \left( \norm{\beta^* - \beta_t}_2^2 \right). 
\]
\fi
%%%%%%%%%%%%%%%%%

\paragraph{Upper Bound} 
We show that a randomized exploration-type algorithm can achieve $O(d\sqrt{T})$ regret. We introduce the following condition. 

% In the first $d\sqrt{T}$ rounds, for each $i\in[d]$, the principal should set the contract as $\beta = \mathbf{e}_i$ for $\sqrt{T}$ rounds. The agents best response $a\in \bbR^d$ should only have a non-zero value in the $i$-th component $a_i$. Hence, the data in these rounds can be used to estimate $\theta_i$. In the remaining rounds, the principal can simply use the estimate $\hat{\theta}$ after the initial $d\sqrt{T}$ rounds, and post the contract as $\beta = \hat{\theta} / k$. 

\begin{condition}
\label{condition:randomExploration}
There exists a distribution $\mathcal{P}$ supported on the feasible contract space $\mathcal{B}$ and that there exists constant $c_1$, such that
\[
\sigma_{\min} (\Ex_{\beta\sim \mathcal{P}}[\beta\beta^T] ) \ge c_1 / d. 
\]
Further, fix any feasible agent's type and letting $a(\beta)$ denote this agent's best response function, we have
\[
\sigma_{\min} (\Ex_{\beta\sim \mathcal{P}}[a(\beta)a(\beta)^T] ) \ge c_1 / d. 
\]
\end{condition}
By sampling a contract from the distribution $\mathcal{P}$, it can be shown that in expectation the increase of $\sigma_{\min}(B_T^T X_T)$ by adding a new data point can be lower bounded, then by applying \Cref{prop:offline-est-1}, the estimates will converge to the true $\theta^*$. 

\begin{remark}
The above condition is in fact relatively mild. The first part requires the sampled contract comes from a diverse enough distribution. As an example, one can choose distribution $\mathcal{P}$ as the uniform distribution over $d$ linearly independent vectors. The second part of the condition is also satisfied as long as the agent's cost function is well-behaved such that variating $\beta$ gives sufficient variance along each direction in the agent's response. 
\end{remark}

\begin{proposition}
\label{prop:learningETC}
Suppose \Cref{condition:randomExploration} holds. There exists an algorithm (\Cref{alg:perturb}) achieving regret $\widetilde{O}(d\sqrt{T})$ with probability $1 - 1/T$. 
\end{proposition}
Instead of working with the actual utility loss, we will be bounding the following quantity, which is the square error of $\beta_t$ each round:
\[
\mathtt{Reg} = \sum_{t=1}^T \norm{\beta^* - \beta_t}_2^2. 
\]
The above can serve as a proxy for the actual cumulative utility loss: $\sum_{t=1}^T \abs{u_t(\beta^*) - u_t(\beta_t)}$. In particular, as long as the function $u_t$ has a bounded second derivative at the global maximum $\beta^*$, the cumulative utility loss can be upper bounded by $O(\tt Reg)$. Then, each term in $\tt Reg$ can be bounded by \Cref{prop:offline-est-1}. 
The detailed proof can be found in \Cref{app:proofETC}. 

One can also consider a ``pure-exploration" problem where, instead of measuring the cumulative utility loss, one can measure the immediate estimation error after $T$ rounds. In this case, one can simply adapt \Cref{alg:perturb} to sample $\beta \sim \mathcal{P}$ in each of the $T$ rounds, and then the estimation error can be bounded by the following. 
\begin{proposition}
Assume Condition \ref{condition:randomExploration} holds. Then there exists a pure-exploration algorithm, that after $T$ rounds with probability $1 - \delta$ achieves: 
\[
\norm{\theta^* - \hat{\theta}_T}_2 \le \tilde{O}(\sqrt{d/T}). 
\]
\end{proposition}

% \Cref{alg:perturb} can be seen as an explore-then-commit type algorithm, where the principal explores using the basis vector set $\set{\mathbf{e}_i}$. Another way to achieve `exploration' is via small random perturbations. Specifically, the principal can add random perturbations with variance $\sigma_0^2$ to each component of $\beta$ in each round. The purpose of this is to ensure the denominator in the GMM estimator can grow as the number of rounds increases, and hence the estimate will become more accurate. 

\paragraph{Lower Bound}
Consider a setting where each agent has cost function $c(a) = \frac{1}{2}\norm{a}_2^2$, and that this cost function is known to the principal. Then, the agent's best response under the contract $\beta$ is exactly $a = \beta$. The principal's expected utility when a contract $\beta$ is posted is then
\[
u(\beta) = \ip{\theta^*}{\beta} - \norm{\beta}_2^2. 
\]
The optimal contract is then $\beta^* = \theta^* / 2$ (this is also implied by the previous uniformity result \Cref{thm:uniformity}, since the agent's cost function is quadratic and homogeneous of degree $2$). The principal now essentially faces a zero-order optimization problem with a quadratic objective function. The contract $\beta$ is the decision variable, and the realized total revenue is the observed variable. By a result in \cite{shamir2013complexity}, the tight regret bound for this problem is $\Theta(d\sqrt{T})$. 

\begin{proposition}
For any (possibly randomized) strategy the principal can adopt for posting contracts, there exists some parameter $\theta^*$ such that the term $\tt Reg$ can be lower bounded by $\Omega(d\sqrt{T})$. 
\end{proposition}

% An algorithm based on random perturbations is proposed to handle the exploration-exploitation tradeoff. Specifically, as the principal obtains an estimate $\hat{\theta}_t$ in round $t$, instead of posting $\beta_t = \hat{\theta}_t / k$, the principal adds a perturbation that decreases as $t$ increases. The algorithm is summarized in \Cref{alg:perturb}. 

% The exploration with random perturbations algorithm is summarized in ???. In this algorithm, the principal obtains the estimate $\hat{\theta}_t$ in each round using the GMM estimator ???. However, instead of setting $\beta_t = \hat{\theta}_t / k$, the principal adds a perturbation term to the `estimated' optimal contract. 

% We are going to analyze the repeated measurement errors in online linear regression. 

% Setting:
% \begin{enumerate}
% \item At each round, a covariate is sampled from a certain distribution
% \item The covariate is not observed. However, two independent copies of noisy observations are observed
% \item The learner must make a prediction to the true label $y$. 
% \end{enumerate}

%% file: sec/repeated.tex
\section{\MakeUppercase{Faster Convergence with Repeated Observations and Diversity}}
\label{sec:repeated}
\begin{algorithm}[t]
\caption{Use Repeated Observation as IV: Pure Exploitation with Diversity}
\label{alg:diverse}
\begin{algorithmic}
\State Parameters: time horizon $T$, number of tasks $d$, minimum eigenvalue $\lambda_0$, failure probability $\delta$
\State Epoch $e$ contains $\abs{\tau_e} = \max\{ d \cdot 2^e, 8 K_0 \ln(d/\delta) / (\lambda_0), 4 \sigma_0^2 d^2 \ln(d^2 / \delta) / (\lambda_0) \}$ rounds
\State In the 0-th epoch with $d$ rounds, post $\beta = \mathbf{e}_i$ for each $i\in[d]$ once and record $\tX_1, X_1, Y_1$
\For {epoch $e = 1, 2, \cdots, \ceil{\log (T/ d)} $} 
    \State Compute $\hat{\theta}_e$ using the GMM estimator from data in previous epoch:
    \[
    \hat{\theta_e} = (\tX_e^\top X_e)^{-1} X_e Y_e
    \]
    \State Post $\beta_t = \hat{\theta}_t / k$ for each round in this epoch
    \State Record the data from this epoch into $\tX_{e+1}, X_{e+1}, Y_{e+1}$
\EndFor
\end{algorithmic}
\end{algorithm}

This section continues the discussion from the previous section, where the principal must learn or estimate $\theta^*$ using observational data. The previous section showed that using an instrumental regression approach and using the contract as the instrumental variable, the principal must balance exploration (ensuring that $\sigma_{\min}(B_t^\top X_t)$ grows as $t$ increases) and exploitation (ensuring $\beta_t$ is close to the true optimal contract). In this section, we show that, under some additional assumptions, the learning rate of the principal can be greatly improved. Informally, we show that when the agents are sufficiently `diverse', and that when repeated observations are available, the principal can achieve a logarithmic regret, even if the principal is using a `pure exploitation' algorithm. We state these as the following two assumptions. 

% We make two key assumptions in this section. The first is that repeated signals are available. The second is that agents are drawn from some stochastic distribution and are sufficiently diverse. 

\begin{assumption}
\emph{(Repeated observations are available.)} We assume that in each round, the principal can observe two sets of signals, denoted by $x_{t}$ and $\tx_{t}$. The signals are conditionally-$\sigma_0^2$-subgaussian vectors and satisfy the following:
		\[
		\Ex[x_{t} | a_t] = \Ex[\tx_{t} | a_t] = a_t. 
		\]
		  They are conditionally independent given the agent's hidden effort $a_t$:
		\[
		x_t\perp \tilde{x}_t | a_t. 
		\]
\end{assumption}
The causal relationship between variables are summarized in \Cref{fig:causalFig}, with the dashed items included. 

\begin{remark}
The assumption requires repeated observations, which is often realistic in practical scenarios where the agent's action produces multiple measurable outcomes. 
We provide several examples where repeated observations naturally arise. \emph{1. Education:} When measuring student skills or abilities, multiple assessments may target the same underlying competencies (e.g., SAT retakes, midterm and final exams), yielding more than one signal per student.
\emph{2. Gig economy and platform work:} A worker providing services may receive multiple customer ratings, each serving as a noisy measurement of their effort or quality.
\emph{3. Lasting effects:} A content creator's effort on video quality may affect both initial view counts and long-term retention metrics, providing two distinct signals of the same underlying effort.
\end{remark}

We assume that agents are drawn from some distribution $\mathcal{D}$, each element in the support of $\mathcal{D}$ represents the cost function of the agent. It is assumed that for all elements in the support of $\mathcal{D}$, the cost function is homogeneous of degree $k$ (i.e., it satisfies \Cref{assump:cost}). Our second part of the assumption below states that agents are sufficiently diverse. 

\begin{assumption}
\label{assumption:diversity}
\emph{(Agents are sufficiently diverse in their talent. )} Fix any $\beta \in \mathcal{B}$, we have that the following holds
		\[
		\lambda_{\min} \Ex_{c(\cdot)\sim \mathcal{D}}[a(\beta; c(\cdot))^\top a(\beta; c(\cdot))] \ge \lambda_0. 
		\]
\end{assumption}

The above assumption is essentially requiring that there is sufficient diverity in the agent's ability to complete different tasks. We give a concrete example satisfying assumption below. 

\begin{example}
Consider a setting where agent's have different abilities in different tasks \cite{thiele2010task}. 
Assume the cost functions of the agents have diagonal quadratic forms. Specifically, the cost function of the agent is parameterized by a vector $\kappa \in (\bbR^d)^+$, so that the cost function of the agent is $c(a) = \frac{1}{2} a^\top \diag(\kappa)^{-1} a$. Then, the agent's best response to a contract $\beta$ is the maximizer of $\ip{\beta}{a} - c(a)$, which is equal to
\[
a = \kappa \odot \beta. 
\]

We assume that the agent cost function in each round, parameterized by $\kappa_t$, is drawn from an unknown distribution. Assume the distribution satisfies the following. 

% The agent in round $t$ has a cost function $\frac{1}{2} w^\top Q_t^{-1} w$, so that the agent's best response under contract $\beta_t$ will be $w_t = Q_t \beta_t$. 

\[
\lambda_{\min} \Ex_{\kappa \sim \cD}[\kappa \kappa^\top] \ge \lambda. 
\]

Then suppose each component of $\beta$ can be lower bounded by some constant, we know that
\begin{align*}
\lambda_{\min} \Ex[a(\beta) a^\top(\beta)] = \Omega(\lambda). 
\end{align*}
Hence the diversity condition is satisfied. 
%\begin{assumption}
%
%\end{assumption}
%
%The key part in the above assumption is the minimum eigenvalue $\lambda_0$, it suggests that the agents' are sufficiently ``diverse" in their types. 
\end{example}

In addition to the above two assumptions, we require a relatively mild boundedness assumption as below. 
\begin{assumption}
Fix any $\beta \in \mathcal{B}$. We have the following for any agent type: $\norm{a(\beta)}_2^2 \le K_0$. Here $a(\beta)$ is the best response to contract $\beta$. 
\end{assumption}

% \[
% x_{t,1} = w_{t} + \eps_{t,1}, \quad x_{t,2} = w_t + \eps_{t,2}. 
% \]
% Here $\eps_{t,1}$ and $\eps_{t,2}$ are zero-mean subgaussian random variables, so that
% \[
% \Ex[x_{t,1}] = \Ex[x_{t,2}] = w_t. 
% \]

%\begin{assumption}
%The two sets of signals are assumed to be (conditionally) independent. Specifically:
%\[
%x_{t} \perp \tx_{t} | w_t. 
%\]
%\end{assumption}

% Further, the noise terms are independent, i.e., 
% \[
% \eps_{t,1} \perp \eps_{t,2}. 
% \]

\subsection{Offline Setting}
The causal relationship is summarized in \Cref{fig:causalFig} with the dashed items included. It can be observed that the second set of observations can serve as an instrumental variable for the other, and that we have:
\[
\Ex[\tx_{t} (y_t - \ip{x_{t}}{\theta})] = \mathbf{0} \ (\in \bbR^d).  
\]
Assume the principal is given some offline data as $((\beta_i, x_i, \tx_i, y_i)_{i=1}^T)$. Then, the principal can use the following GMM estimator to obtain an estimate $\hat{\theta}$:
\begin{equation}
\hat{\theta}_T = (\tX_T^\top X_T)^{-1} \tX_T^\top Y_T. \label{eq:GMM_repeated}
\end{equation}
Here $X_T, \tX_T$ are a $T\times d$ matrices where row $t$ is $x_t$ or $\tx_t$ respectively; $Y_T$ is a vector with the $t$-th entry being $y_t$. 

We state the convergence result for the GMM estimator above. 

\begin{proposition}
\label{prop:offline-est-repeated}
	With probability $1-\delta$, 
	\[
	\norm{\theta - \theta^*}_2\le \frac{\sqrt{dT\log(dT/\delta)}}{\sigma_{\min}(\tX_T^\top X_T)}. 
	\]
\end{proposition}
The proof is straightforward and similar to that of $\Cref{prop:offline-est-1}$. 

%Similar to the previous section, when offline data are avialable, the principal can have a estimation error bound on the paramter $\theta^*$. 

%\begin{proposition}
%	Given a sequence of offline data, using the GMM estimation gives
%	\[
%	\norm{???} \le ???. 
%	\]
%\end{proposition}
%\begin{remark}
%	content...
%\end{remark}

\subsection{Online Setting: Pure exploitation algorithm}

We now turn to the online setting, where the principal must learn the optimal contract in real time. In this setting, the proposed algorithm relies directly on the estimator introduced above, without the need for explicit exploration. To reduce computational cost, the algorithm proceeds in epochs whose lengths grow geometrically. At the beginning of each epoch~$e$, the principal updates the estimate $\hat{\theta}_e$ using data collected in previous epochs, and then posts contracts based on this estimate throughout the epoch (see \Cref{alg:diverse} for details).

% A key distinction from the algorithms in the previous section is that here the principal does not encounter a conventional exploration--exploitation tradeoff. The inherent diversity across agents effectively provides the necessary exploration, allowing the algorithm to focus on exploitation while still ensuring efficient learning.

A key distinction from the algorithms in the previous section is that here the principal does not encounter a conventional exploration–exploitation tradeoff. The inherent diversity across agents effectively provides the necessary exploration, allowing the algorithm to focus on exploitation while still ensuring efficient learning.

% We again turn to the online setting where the principal must learn the optimal contract. The algorithm will simply use the estimator above, without regard to `exploration'. In the implementation, the algorithm proceeds in epochs with epochs doubling in length to save computation. In each epoch $e$, the principal sets the contract using the estimation $\hat{\theta}_e$ according to the above formula using data from the previous epochs. See \Cref{alg:diverse} for details. 

% The algorithm will simply use the estimate above and proceed in epochs, where the length of the epochs are doubling in length. In each epoch, the principal sets the contract using the estimated value for $\theta$, where $\theta$ is estimated according to the above formula using data from the previous epoch. With a slight abuse of notation, denote ${X}_e, \Tilde{X}+e_t$ as the matrix containing the first and second set of observations in each epoch respectively. For each round in epoch $e$, the following estimator will be used
% \[
% \theta_t = \dots
% \]
% The contract in each round $t$ will simply be $\theta_t / 2$. 

% Different from the algorithms in the previous section, the algorithm does not face the exploration-exploitation tradeoff, as the `diversity' of the agents already acts as a form of exploration. 

% \begin{proposition}
% Suppose the optimal contract is $\beta^*$. Further suppose $u''(\beta^*) < 0$. For any $\beta$, 
% \[
% u(\beta^*) - u(\beta) < O(\norm{\beta - \beta^*}_2^2). 
% \]
% \end{proposition}

\begin{theorem}
\label{thm:diversity}
With probability $1 - 1/T$, the cumulative regret of the algorithm is $\tO(d / \lambda_0^2)$. 
\end{theorem}
\begin{proof}[Proof Sketch. ]
Consider an epoch $e$ with length $\abs{\tau_e}$. The minimum singular value of $(\tX_e^\top X_e)$ is on the order of $\Omega(\abs{\tau_e} \lambda_0)$. At the same time, the term $\norm{\tX_e^\top Y_e}$ can be upper bounded by $\tO(\sqrt{d \abs{\tau_e}})$. Hence, for each round in the epoch, the squared estimation error can be upper bounded by:
\[
\norm{\hat{\theta}_e - \theta}_2^2 \le \tO(d / (\abs{\tau_e} \lambda_0^2)). 
\]
Therefore, the regret in each epoch is on the order of $\tO(d / \lambda_0^2)$. There are a total of $O(\log T)$ epochs, therefore the total regret is $\tO(d / \lambda_0^2)$. 
\end{proof}

\begin{remark}
\Cref{thm:diversity} relies critically on the minimum eigenvalue condition in \Cref{assumption:diversity}. In practice, one can verify whether this condition holds by examining the singular values of the matrix $\tilde{X}_T^\top X_T$ constructed from observed data. By \Cref{prop:offline-est-repeated}, if the minimum singular value is large, convergence is fast and pure exploitation performs well. Conversely, if the minimum singular value is small, the diversity condition may not hold in practice, and the principal should introduce explicit exploration into the algorithm.
An interesting direction for future work is to design a fully adaptive algorithm that monitors the observed singular values and automatically transitions between exploration and exploitation. Such an algorithm would combine the efficiency of pure exploitation when agent diversity is sufficient with the robustness of explicit exploration when it is not.
\end{remark}

%% file: sec/conclusion.tex
\section{CONCLUSION}

In this work, we studied the multitasking principal–agent problem under moral hazard. We first established a uniformity result, showing that the optimal contract depends only on the degree of homogeneity. We then demonstrated how an instrumental regression approach, leveraging the generalized method of moments (GMM) estimator, can be used to recover the unknown parameters and identify the optimal contract. 

Beyond this specific setting, our results suggest that instrumental methods can be applied more broadly in interactive scenarios whenever moral hazard is present. In policy-making scenarios where hidden actions influence true outcomes, the policy itself may serve as a valid instrumental variable. Such interactions can also be modeled as Stackelberg games, where the leader is the policy-maker, and the follower is an individual or population of the policy target. This perspective opens the door to developing a more general framework for data-driven decision-making in the presence of moral hazard.

\szcomment{Ask claude to rewrite. }
% Finally, we note the broader connection between machine learning and econometrics. While both fields are fundamentally concerned with analyzing data, they have evolved with different emphases. While machine learning and econometrics share fundamental goals, they have developed complementary methodological toolkits. Our work illustrates how classical econometric methods (GMM, instrumental variables) can be integrated into online learning systems to handle strategic environments with hidden actions—a setting increasingly relevant as ML is deployed in policy-making contexts." In addition, we will acknowledge the recent work on causality and identifiability in recent ML research themes. Yet, in many real-world applications, effective data-driven solutions require insights from both traditions. Our work illustrates this synergy in the specific problem of multitasking principal-agent, and shows how the GMM estimator, a well-studied tool in econometrics, can be applied within machine learning and online learning systems. 

Finally, we highlight the broader connection between machine learning and econometrics. While both fields are fundamentally concerned with analyzing data, they have evolved with different emphases and developed complementary methodological toolkits. Our work demonstrates how classical econometric tools---specifically, instrumental variable regression and the generalized method of moments---can be integrated into online learning algorithms to address moral hazard. 

% More broadly, our results suggest that the methodological divide between machine learning and econometrics is not fundamental but rather reflects different historical priorities. As data-driven decision-making expands into domains characterized by strategic behavior and hidden actions, we anticipate that techniques from both traditions will become increasingly intertwined. We hope this work serves as a step toward a more unified framework for learning in the presence of moral hazard.

%% file: sec/app_naiveEst.tex
\section{\MakeUppercase{Concentration Tools}}
\label{app:concentration}
The below lemma gives a deviation bound on random vectors. 
\begin{lemma} [Adapted from \cite{kannan2018smoothed}]
Let $\gamma_1, \dots, \gamma_T$ be independent $\sigma$-subgaussian random variables. Let $v_1, \dots, v_T$ be vectors in $\mathbb{R}^d$ with each $v_T$ chosen arbitrarily as a function of $(v_1, \gamma_1), \dots, (v_{t-1}, \gamma_{t-1})$ subject to $\norm{v_t} \le C$. Then with probability at least $1-\delta$, 
\[
\norm{\sum_{t=1}^T \gamma_t v_t} \le \sqrt{2d C\sigma T \log (Td / \delta) }. 
\]
\end{lemma}

The below lemma bounds the eigenvalues of random matrices; it is adapted from \cite{tropp2012user}. 
\szcomment{This lemma bounds the min eigenvalue of the sum of symmetric matrices}
\begin{lemma}[Matrix Chernoff, adapted from \cite{tropp2012user}]
Consider a finite sequence $\set{X_k}$ of independent, random, symmetric matrix with dimension $d$. Assume that each matrix satisfies
\[
\lambda_{\min} X_k \ge 0, \lambda_{\max} X_k \le R. 
\]
Define
\begin{align*}
\mu_{\min} := \lambda_{\min} (\sum_k \Ex X_k), \\
\mu_{\max}:= \lambda_{\max} (\sum_k \Ex X_k). 
\end{align*}
Then
\[
\Pr[ \lambda_{\min} (\sum_k X_k) \le (1 - \delta) \mu_{\min} ] \le d\cdot \left[ \frac{e^{-\delta}}{(1-\delta)^{1-\delta}} \right]^{\mu_{\min} / R}. 
\]
\end{lemma}
The following Corollary may be easier to use than the above Lemma. 
\begin{cor}
\label{cor:tropp}
Using the same setup as above, we have 
\[
\Pr[ \lambda_{\min} (\sum_k X_k) \le (1 - \delta) \mu_{\min} ] \le d \exp(-\delta^2 \mu_{\min} / (2R)). 
\]
\end{cor}

\section{PROOF OF PROPOSITION \ref{prop:learningETC}}
\label{app:proofETC}
\begin{proof}
We focus on lower bounding the quantity $\sigma_{\min}(B_\tau^\top X_\tau)$. To start, note 
\[
\sigma_{\min}(B_\tau^\top X_\tau) \ge \sigma_{\min}(B_\tau) \sigma_{\min}(X_\tau). 
\]
Further, 
\begin{align*}
    \sigma_{\min}^2(B_\tau) &= \lambda_{\min}(\sum_{t=1}^\tau \beta_t \beta_t^\top) \\
    \sigma_{\min}^2(X_\tau) &= \lambda_{\min}(\sum_{t=1}^\tau x_t x_t^\top) \\
\end{align*}
By \Cref{cor:tropp} and \Cref{condition:randomExploration}, with probability $1 - 2\delta$, we have
\begin{align*}
\sigma_{\min}^2(B_\tau) > \Omega(\tau \log(1/\delta) / d), \\
\sigma_{\min}^2(X_\tau) > \Omega(\tau \log (1/\delta) / d). 
\end{align*}
Hence, we have
\[
\sigma_{\min} (B_\tau^\top X_\tau) \ge \Omega(\tau / d) = \Omega(\sqrt{T}). 
\]
Therefore with probability $1-\delta$, 
\[
\norm{\hat{\theta} - \theta^*} \le \sqrt{d \log(1/\delta) / \sqrt{T}}. 
\]
\end{proof}

%% file: sec/app_diversity.tex
\section{PROOF OF THEOREM \ref{thm:diversity}}

Recall the estimator in each epoch
\[
\hat{\theta}_e = \left( \tX_e^\top X_e \right)^{-1} \tX_e Y_e. 
\]
Denote $\gamma_t = y_t - \ip{x_t}{\theta}$, and collect $\gamma_t$ into the column vector $\Gamma_e$. Recall that by \Cref{prop:offline-est-repeated}, we have 
\begin{align}
\label{eq:GMMepoch}
\norm{\hat{\theta}_e - \theta}_2 \le \frac{\norm { \tX_e \Gamma_e} }{\sigma_{\min} (\tX_e^\top X_e)}. 
\end{align}
To bound the estimation error, the following will lower bound $\sigma_{\min} (\tX_e^\top X_e)$ and upper bound $\norm{\tX_e \Gamma_e}$ separately. 

% \subsection{Lower Bounding $\sigma_{\min} (\tX_e^\top X_e)$}
% \begin{theorem}
% The cumulative regret of the algorithm is $\tO(d)$. 
% \end{theorem}
% \begin{proof}
% Consider the estimate at the end of epoch $e$. The estimate $\theta_e$. 
% \begin{align*}
% \theta_e - \theta &= (\tildeX_e X_e)^{-1} X' Y - \theta \\
% &= (\tildeX_e X_e)^{-1} X' (X\theta + \eta) - \theta \\
% &= (\tildeX_e X_e)^{-1} \tildeX \eta
% \end{align*}
% \end{proof}

\begin{lemma}
With high probability $\delta$, $\lambda_{\min}(\sum_{t\in \tau_e} a_t a_t^\top ) \ge \abs{\tau_e} \lambda_0 / 2$. 
\end{lemma}
\begin{proof}
The matrices $\set{W_t}$, where $W_t = a_t a_t^\top$ satisfies the condition in the Lemma with $\lambda_{\max} W_t = \norm{a_t}_2^2 \le K_0$. 

% Further, 
% \begin{align*}
% \Ex [W_t] &= \Ex [w_t w_t^\top] \\
% &= \Ex[\diag(\kappa_t) \theta_t \theta_t^\top \diag(\kappa_t)] \\
% &= \Ex[\Theta_t \kappa_t \kappa_t^\top \Theta_t] \\
% &= \Theta_t \Ex[\kappa_t \kappa_t^\top] \Theta_t \\
% \end{align*}

By the diversity condition, 
\begin{align*}
\lambda_{\min} (\Ex W_t) \ge \lambda_0. 
\end{align*}

By \cref{cor:tropp}, 
\begin{align*}
\Pr[\lambda_{\min} (\sum_k W_k) \le \abs{\tau_e} \lambda_0 / 2 ] &\le d\cdot \exp(- \abs{\tau_e}  \lambda_0 / (8 K_0)) \\
&\le \delta
\end{align*}
when 
\[
\abs{\tau_e} \ge \frac{8 K_0}{\lambda_0} \cdot \ln(d /\delta)
\]
\end{proof}

\begin{lemma}
\label{lemma:denom}
With probability $1 - 2\delta$, for each epoch $e$ with length as specified in \Cref{alg:diverse}, $\sigma_{\min}(\tildeX_e^\top X_e) \ge \abs{\tau_e} \lambda_0 / 4$. 
\end{lemma}
\begin{proof}
In the following denote $\eps_t = x_t - a_t, \tilde{\eps}_t = \tx_t - a_t$. 
\begin{align*}
\tildeX_e^\top X_e &= \sum_{t\in \tau_e} x_t \tx_t^\top \\
&= \sum_{t\in \tau_e} (a_t + \eps_t) (a_t + \tilde{\eps}_t)^\top \\
&= \sum_{t\in \tau_e} a_t^\top a_t + \eps_t^\top a_t + a_t \eps_t + \eps_t^\top \tilde{\eps}_t
\end{align*}

The previous lemma showed that with probability $1 - \delta$,
\[
\lambda_{\min} (\sum_{t\in \tau_e} a_t^\top a_t) \ge \abs{\tau_e} \lambda_0 / 2. 
\]

The remaining term 
\[
\sum_{t\in\tau_e} \eps_t a_t^\top + a_t \tilde{\eps}_t^\top + \eps \tilde{\eps}^\top
\]
is a random matrix with mean $\mathbf{0}$. By a standard concentration inequality, the absolute value of each entry in the matrix can be upper-bounded as $\sqrt{\abs{\tau_e} \sigma_0^2 \ln (d^2 / \delta)}$
with probability $\delta$. Then, since the minimum singular value must be upper bounded by the Frobenius norm: 
\[
\sigma_{\min} \left( \sum_{t\in\tau_e} \eps_t w_t^\top + w_t \tilde{\eps}_t^\top + \eps \tilde{\eps}^\top \right) \le \sqrt{\abs{\tau_e}\sigma_0^2 d^2 \ln(d^2 / \delta)}. 
\]

The conclusion is that with probability $1 - 2\delta$, 
\begin{align*}
\sigma_{\min} (\tX_e^\top X_e) &\ge \abs{\tau_e} \lambda_0 / 2 - \sqrt{\abs{\tau_e} \sigma_0^2 d^2 \ln (d^2 / \delta)} \\
&\ge \abs{\tau_e} \lambda_0 / 4. 
\end{align*}
\szcomment{add sometinghere, condition on the length of $\tau_e$}
This finishes the proof. 
\end{proof}

%\subsection{Upper Bounding $\norm{\tX_e \Gamma_e}$}
% The below Lemma is a standard concentration result, and can be found in, e.g., \cite{???}. 
% \begin{lemma}
% Let $\gamma_1, ... \gamma_t$ be sub-gaussian variables with variance $\sigma_0^2$. Let $x_1, \dots, x_t$ be vectors in $\mathbb{R}^d$ with each $x_{\tau}$ chosen as a function of $(x_1, gamma_1), \dots, (x_{\tau-1}, \gamma_{\tau_1})$ subject to $\norm{x_{\tau}} \le r$. With probability $1 - \delta$, 
% \[
% \norm{} <= \sqrt{2rs d t \ln(d t /\delta)}. 
% \]
% \end{lemma}

\begin{lemma}
\label{lemma:numerator}
%Recall $\gamma_t = y_t - \ip{x_t}{\theta}$. 
With probability $1 - \delta$, $\norm{\tX_e \Gamma_e} \le \sqrt{2 \sigma_0 K_0 d \abs{\tau_e} \ln (\abs{\tau_e} d / \delta)}$. 
\end{lemma}
\begin{proof}
Recall $\gamma_t = y_t - \ip{x_t}{\theta}$, and $\Gamma_e$ is the column vector collecting the variables $\gamma_t$. Note that $\gamma_t \tx_t$ is a zero-mean, $\sigma_0 K_0$-subgaussian random vector and independent conditioned on all previous steps. The result then follows from the standard concentration inequality (see \Cref{app:concentration}). 
% \begin{align*}
% y_t &= \ip{w_t}{\theta} + \eta \\
% &= \ip{x_t - \eps_t}{\theta} + \eta \\
% &= \ip{x_t}{\theta} - \ip{\eps_t}{\theta} + \eta
% \end{align*}
\end{proof}

%\subsection{Completing the Proof}
\begin{theorem}
With probability $1 - 1/T$, the total regret is $\tO(d / \lambda_0^2)$. 
\end{theorem}
\begin{proof}
By \Cref{lemma:numerator}, \Cref{lemma:denom}, and \Cref{eq:GMMepoch}, with probability $1 - 3\delta$, the regret in each epoch can be upper bounded as $\abs{\tau_e}\cdot \tO(d / (\abs{\tau_e} \lambda_0^2)) = \tO(d / \lambda_0^2)$. 
Choose some appropriate $\delta$, e.g., $\delta < O(1 / T^2)$. Since there are at most $\log T$ epochs, the total regret will be $\tO(d / \lambda_0^2)$. 
\end{proof}

%% file: sec/experiments.tex
\section{NUMERICAL EXPERIMENTS}
% \begin{center}
% \begin{figure}[!htbp]
% \centering
% \begin{subfigure}{.5\textwidth}
%   \centering
%   \includegraphics[width=1.0\linewidth]{figures/GMM_iv1.png}
%   \caption{}
%   \label{fig:sub1}
% \end{subfigure}%
% \begin{subfigure}{.5\textwidth}
%   \centering
%   \includegraphics[width=1.0\linewidth]{figures/GMM_repeated_new.png}
%   \caption{}
%   \label{fig:sub2}
% \end{subfigure}
% \caption{The estimation error $\norm{\hat{\theta} - \theta^*}$ as sample size increases. (a): Using the contract as the instrumental variable and estimating $\hat{\theta}$ as $\hat{\theta} = (B_T X_T)^{-1} X_T Y_T$.  (\Cref{prop:offline-est-1}). (b): Using repeated observations as the instrumental variable and estimating $\hat{\theta}$ as $\hat{\theta} = (\tilde{X}_T X_T)^{-1} \tilde{X}_T Y_T$ (\Cref{prop:offline-est-repeated}). In both cases, we observe the estimation error decreases at a rate around $\Theta(1/\sqrt{T})$ (where $T$ is the sample size). }
% \label{fig:GMM_experiments}
% \end{figure}
% \end{center}

We perform numerical simulations that test the performance of the generalized method of moments based estimator.\footnote{Code can be found at \url{https://anonymous.4open.science/r/GMM_learning-8E3F/README.md}. } We choose $d = 5$ and $\theta^* = [1,2,3,4,5]$. We choose the agent's cost function as 
\[
c(a) = \sum_{i=1}^d \kappa_i a_i^2, 
\]
where each $\kappa_i$ is sampled uniformly from the set $\set{1,10}$. We study how the estimation error decreases as more samples are received. 

In the first setting, we use the contract as the instrumental variable and use the estimator as in Eq.~\Cref{eq:GMM_v1}. In the second setting, we assume repeated observations are available and use the estimator as in Eq.~\Cref{eq:GMM_repeated}. We plot the relationship in \Cref{fig:GMM_experiments}. As we can observe, the estimation error $\norm{\hat{\theta} - \theta^*}$ decreases at a rate around $O(1/\sqrt{T})$.

\clearpage
\begin{figure}[H]  % try [!t] or [!htb], avoid 'p' to prevent float pages
  \centering
  \begin{subfigure}{0.49\textwidth}
    \centering
    \includegraphics[width=\linewidth]{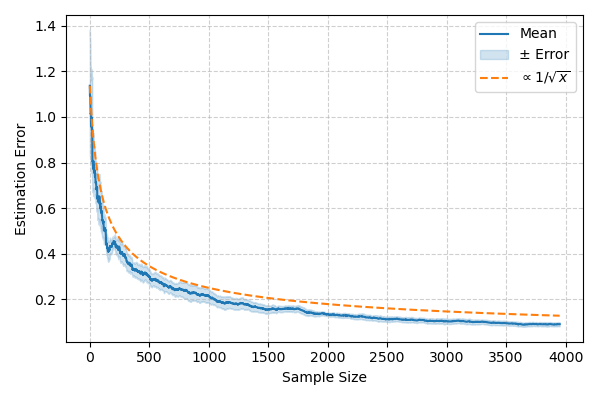}
    \caption{}
    \label{fig:sub1}
  \end{subfigure}\hfill
  \begin{subfigure}{0.49\textwidth}
    \centering
    \includegraphics[width=\linewidth]{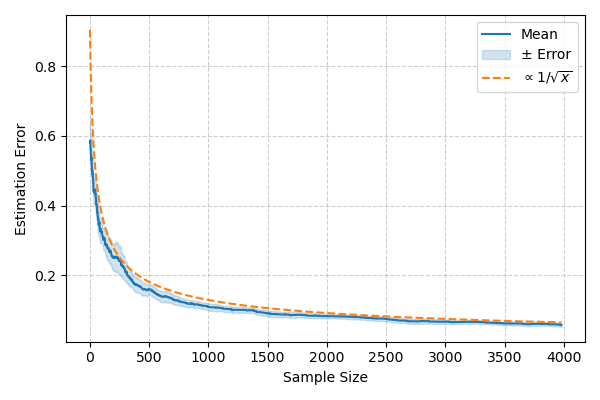}
    \caption{}
    \label{fig:sub2}
  \end{subfigure}
  \caption{The estimation error $\norm{\hat{\theta} - \theta^*}$ as sample size increases. (a): Using the contract as the instrumental variable and estimating $\hat{\theta}$ as $\hat{\theta} = (B_T X_T)^{-1} X_T Y_T$.  (\Cref{prop:offline-est-1}). (b): Using repeated observations as the instrumental variable and estimating $\hat{\theta}$ as $\hat{\theta} = (\tilde{X}_T X_T)^{-1} \tilde{X}_T Y_T$ (\Cref{prop:offline-est-repeated}). In both cases, we observe the estimation error decreases at a rate around $\Theta(1/\sqrt{T})$ (where $T$ is the sample size).}
  \label{fig:GMM_experiments}
\end{figure}

%% file: AISTATS2026PaperPack/aistats_main.bbl
\begin{thebibliography}{}

\bibitem[Ashrafian and Darzi, 2018]{ashrafian2018transforming}
Ashrafian, H. and Darzi, A. (2018).
\newblock Transforming health policy through machine learning.
\newblock {\em PLoS Medicine}, 15(11):e1002692.

\bibitem[Bastani et~al., 2021]{bastani2021mostly}
Bastani, H., Bayati, M., and Khosravi, K. (2021).
\newblock Mostly exploration-free algorithms for contextual bandits.
\newblock {\em Management Science}, 67(3):1329--1349.

\bibitem[Bhattacharyya and Lafontaine, 1995]{bhattacharyya1995double}
Bhattacharyya, S. and Lafontaine, F. (1995).
\newblock Double-sided moral hazard and the nature of share contracts.
\newblock {\em The RAND Journal of Economics}, pages 761--781.

\bibitem[Carroll, 2015]{carroll2015robustness}
Carroll, G. (2015).
\newblock Robustness and linear contracts.
\newblock {\em American Economic Review}, 105(2):536--563.

\bibitem[Della~Vecchia and Basu, 2023]{della2023online}
Della~Vecchia, R. and Basu, D. (2023).
\newblock Online instrumental variable regression: Regret analysis and bandit feedback.
\newblock {\em arXiv preprint arXiv:2302.09357}.

\bibitem[Dong et~al., 2018]{dong2018strategic}
Dong, J., Roth, A., Schutzman, Z., Waggoner, B., and Wu, Z.~S. (2018).
\newblock Strategic classification from revealed preferences.
\newblock In {\em Proceedings of the 2018 ACM Conference on Economics and Computation}, pages 55--70.

\bibitem[Duetting et~al., 2024]{duetting2024multi}
Duetting, P., Ezra, T., Feldman, M., and Kesselheim, T. (2024).
\newblock Multi-agent combinatorial contracts.
\newblock {\em arXiv preprint arXiv:2405.08260}.

\bibitem[D{\"u}tting et~al., 2022]{dutting2022combinatorial}
D{\"u}tting, P., Ezra, T., Feldman, M., and Kesselheim, T. (2022).
\newblock Combinatorial contracts.
\newblock In {\em 2021 IEEE 62nd Annual Symposium on Foundations of Computer Science (FOCS)}, pages 815--826. IEEE.

\bibitem[D{\"u}tting et~al., 2019]{dutting2019simple}
D{\"u}tting, P., Roughgarden, T., and Talgam-Cohen, I. (2019).
\newblock Simple versus optimal contracts.
\newblock In {\em Proceedings of the 2019 ACM Conference on Economics and Computation}, pages 369--387.

\bibitem[Fuller, 2009]{fuller2009measurement}
Fuller, W.~A. (2009).
\newblock {\em Measurement error models}.
\newblock John Wiley \& Sons.

\bibitem[Fuster et~al., 2022]{fuster2022predictably}
Fuster, A., Goldsmith-Pinkham, P., Ramadorai, T., and Walther, A. (2022).
\newblock Predictably unequal? the effects of machine learning on credit markets.
\newblock {\em The Journal of Finance}, 77(1):5--47.

\bibitem[Google, 2024]{YouTubeShare}
Google (2024).
\newblock Youtube partner earnings overview.

\bibitem[Greene, 2008]{alma99653821012205899}
Greene, W.~H. (2008).
\newblock {\em Econometric analysis}.
\newblock Pearson/Prentice Hall, Upper Saddle River, N.J, 6th ed. edition.

\bibitem[Guruganesh et~al., 2024]{guruganesh2024contracting}
Guruganesh, G., Kolumbus, Y., Schneider, J., Talgam-Cohen, I., Vlatakis-Gkaragkounis, E.-V., Wang, J.~R., and Weinberg, S.~M. (2024).
\newblock Contracting with a learning agent.
\newblock {\em arXiv preprint arXiv:2401.16198}.

\bibitem[Hardt et~al., 2016]{hardt2016strategic}
Hardt, M., Megiddo, N., Papadimitriou, C., and Wootters, M. (2016).
\newblock Strategic classification.
\newblock In {\em Proceedings of the 2016 ACM conference on innovations in theoretical computer science}, pages 111--122.

\bibitem[Harris et~al., 2022]{harris2022strategic}
Harris, K., Ngo, D. D.~T., Stapleton, L., Heidari, H., and Wu, S. (2022).
\newblock Strategic instrumental variable regression: Recovering causal relationships from strategic responses.
\newblock In {\em International Conference on Machine Learning}, pages 8502--8522. PMLR.

\bibitem[Hilbert et~al., 2021]{hilbert2021machine}
Hilbert, S., Coors, S., Kraus, E., Bischl, B., Lindl, A., Frei, M., Wild, J., Krauss, S., Goretzko, D., and Stachl, C. (2021).
\newblock Machine learning for the educational sciences.
\newblock {\em Review of Education}, 9(3):e3310.

\bibitem[Ho et~al., 2014]{ho2014adaptive}
Ho, C.-J., Slivkins, A., and Vaughan, J.~W. (2014).
\newblock Adaptive contract design for crowdsourcing markets: Bandit algorithms for repeated principal-agent problems.
\newblock In {\em Proceedings of the fifteenth ACM conference on Economics and computation}, pages 359--376.

\bibitem[Holmstr{\"o}m, 1979]{holmstrom1979moral}
Holmstr{\"o}m, B. (1979).
\newblock Moral hazard and observability.
\newblock {\em The Bell journal of economics}, pages 74--91.

\bibitem[Holmstrom, 1982]{holmstrom1982moral}
Holmstrom, B. (1982).
\newblock Moral hazard in teams.
\newblock {\em The Bell journal of economics}, pages 324--340.

\bibitem[Holmstrom and Milgrom, 1991]{holmstrom1991multitask}
Holmstrom, B. and Milgrom, P. (1991).
\newblock Multitask principal--agent analyses: Incentive contracts, asset ownership, and job design.
\newblock {\em The Journal of Law, Economics, and Organization}, 7(special\_issue):24--52.

\bibitem[Horowitz and Rosenfeld, 2023]{horowitz2023causal}
Horowitz, G. and Rosenfeld, N. (2023).
\newblock Causal strategic classification: A tale of two shifts.
\newblock In {\em International Conference on Machine Learning}, pages 13233--13253. PMLR.

\bibitem[Hurley and Adebayo, 2016]{hurley2016credit}
Hurley, M. and Adebayo, J. (2016).
\newblock Credit scoring in the era of big data.
\newblock {\em Yale JL \& Tech.}, 18:148.

\bibitem[Jain et~al., 2023]{jain2023adaptive}
Jain, S., Pattanayak, K., Krishnamurthy, V., and Berry, C. (2023).
\newblock Adaptive eccm for mitigating smart jammers.
\newblock In {\em ICASSP 2023-2023 IEEE International Conference on Acoustics, Speech and Signal Processing (ICASSP)}, pages 1--5. IEEE.

\bibitem[Kannan et~al., 2018]{kannan2018smoothed}
Kannan, S., Morgenstern, J.~H., Roth, A., Waggoner, B., and Wu, Z.~S. (2018).
\newblock A smoothed analysis of the greedy algorithm for the linear contextual bandit problem.
\newblock {\em Advances in neural information processing systems}, 31.

\bibitem[Kleinberg, 2004]{kleinberg2004nearly}
Kleinberg, R. (2004).
\newblock Nearly tight bounds for the continuum-armed bandit problem.
\newblock {\em Advances in Neural Information Processing Systems}, 17.

\bibitem[Lattimore and Szepesv{\'a}ri, 2020]{lattimore2020bandit}
Lattimore, T. and Szepesv{\'a}ri, C. (2020).
\newblock {\em Bandit algorithms}.
\newblock Cambridge University Press.

\bibitem[Miller et~al., 2020]{miller2020strategic}
Miller, J., Milli, S., and Hardt, M. (2020).
\newblock Strategic classification is causal modeling in disguise.
\newblock In {\em International Conference on Machine Learning}, pages 6917--6926. PMLR.

\bibitem[Perdomo et~al., 2020]{perdomo2020performative}
Perdomo, J., Zrnic, T., Mendler-D{\"u}nner, C., and Hardt, M. (2020).
\newblock Performative prediction.
\newblock In {\em International Conference on Machine Learning}, pages 7599--7609. PMLR.

\bibitem[Qin et~al., 2022]{qin2022reinforcement}
Qin, Z.~T., Zhu, H., and Ye, J. (2022).
\newblock Reinforcement learning for ridesharing: An extended survey.
\newblock {\em Transportation Research Part C: Emerging Technologies}, 144:103852.

\bibitem[Roth et~al., 2016]{roth2016watch}
Roth, A., Ullman, J., and Wu, Z.~S. (2016).
\newblock Watch and learn: Optimizing from revealed preferences feedback.
\newblock In {\em Proceedings of the forty-eighth annual ACM symposium on Theory of Computing}, pages 949--962.

\bibitem[Shamir, 2013]{shamir2013complexity}
Shamir, O. (2013).
\newblock On the complexity of bandit and derivative-free stochastic convex optimization.
\newblock In {\em Conference on Learning Theory}, pages 3--24. PMLR.

\bibitem[Shavit et~al., 2020]{shavit2020causal}
Shavit, Y., Edelman, B., and Axelrod, B. (2020).
\newblock Causal strategic linear regression.
\newblock In {\em International Conference on Machine Learning}, pages 8676--8686. PMLR.

\bibitem[Sivakumar et~al., 2022]{sivakumar2022smoothed}
Sivakumar, V., Zuo, S., and Banerjee, A. (2022).
\newblock Smoothed adversarial linear contextual bandits with knapsacks.
\newblock In {\em International Conference on Machine Learning}, pages 20253--20277. PMLR.

\bibitem[Thiele, 2010]{thiele2010task}
Thiele, V. (2010).
\newblock Task-specific abilities in multi-task principal--agent relationships.
\newblock {\em Labour Economics}, 17(4):690--698.

\bibitem[Tropp, 2012]{tropp2012user}
Tropp, J.~A. (2012).
\newblock User-friendly tail bounds for sums of random matrices.
\newblock {\em Foundations of computational mathematics}, 12:389--434.

\bibitem[Yu et~al., 2022]{yu2022strategic}
Yu, M., Yang, Z., and Fan, J. (2022).
\newblock Strategic decision-making in the presence of information asymmetry: Provably efficient rl with algorithmic instruments.
\newblock {\em arXiv preprint arXiv:2208.11040}.

\bibitem[Zhu et~al., 2022]{Zhu2022TheSC}
Zhu, B., Bates, S., Yang, Z., Wang, Y., Jiao, J., and Jordan, M.~I. (2022).
\newblock The sample complexity of online contract design.
\newblock {\em Proceedings of the 24th ACM Conference on Economics and Computation}.

\bibitem[Zuo, 2024]{zuo2024harnessing}
Zuo, S. (2024).
\newblock Harnessing the continuous structure: Utilizing the first-order approach in online contract design.
\newblock {\em arXiv preprint arXiv:2403.07143}.

\end{thebibliography}
